%% file: main.tex
\documentclass[10pt,twocolumn,letterpaper]{article}
\usepackage{standalone}
\usepackage{titling}
\usepackage[pagenumbers]{cvpr} 

\usepackage{graphicx}
\usepackage{amsmath}
\usepackage{amssymb}
\usepackage{booktabs}
\usepackage{bm}
\usepackage{color}
\usepackage{amsthm}
\newtheorem{theorem}{Theorem}
\newtheorem{lemma}{Lemma}
\newenvironment{sproof}{%
	\proof}{\endproof}

\usepackage{multirow}
\usepackage[accsupp]{axessibility} 

\usepackage[toc,page,titletoc]{appendix}

\usepackage[pagebackref,breaklinks,colorlinks]{hyperref}

\usepackage[capitalize]{cleveref}
\crefname{section}{Sec.}{Secs.}
\Crefname{section}{Section}{Sections}
\Crefname{table}{Table}{Tables}
\crefname{table}{Tab.}{Tabs.}

\def\cvprPaperID{6460} 
\def\confName{CVPR}
\def\confYear{2022}

\begin{document}

\title{The Devil is in the Pose: Ambiguity-free 3D Rotation-invariant \\ Learning via Pose-aware Convolution}
\author{Ronghan Chen\textsuperscript{1,2,3} \quad \quad \quad
Yang Cong\textsuperscript{1,2}\thanks{The corresponding author is Prof. Yang Cong.}
\\
\textsuperscript{1}State Key Laboratory of Robotics, Shenyang Institute of Automation, Chinese Academy of Sciences
\thanks{This work is supported in part by the National Key Research and Development Program of China under Grant 2019YFB1310300 and the National Nature Science Foundation of China under Grant 62127807.}
\\
\textsuperscript{2}Institutes for Robotics and Intelligent Manufacturing, Chinese Academy of Sciences\\
\textsuperscript{3}University of Chinese Academy of Sciences\\
{\tt\small chenronghan@sia.cn, congyang81@gmail.com}
}
\maketitle

\begin{abstract}
	
	Rotation-invariant (RI) 3D deep learning methods suffer performance degradation as they typically design RI representations as input that lose critical global information comparing to 3D coordinates. 
	Most state-of-the-arts address it by incurring additional blocks or complex global representations in a heavy and ineffective manner. 
	In this paper, we reveal that the global information loss stems from an unexplored \textbf{pose information loss problem},
	which can be solved more efficiently and effectively as we only need to restore more lightweight local pose in each layer, and the global information can be hierarchically aggregated in the deep networks without extra efforts.
	To address this problem, we develop a \underline{P}ose-\underline{a}ware \underline{R}otation \underline{I}nvariant \underline{Conv}olution (\emph{i.e.}, \textbf{PaRI-Conv}), which dynamically adapts its kernels based on the relative poses. 
	To implement it, we propose an Augmented Point Pair Feature (APPF) to fully encode the RI relative pose information, and a factorized dynamic kernel for pose-aware kernel generation, which can further reduce the computational cost and memory burden by decomposing the kernel into a shared basis matrix and a pose-aware diagonal matrix. 
	Extensive experiments on shape classification and part segmentation tasks show that our PaRI-Conv surpasses the state-of-the-art RI methods while being more compact and efficient. 
	
\end{abstract}

\section{Introduction}

\begin{figure}[t]
	\centering
			\vspace{-5pt}
	\includegraphics[width=0.48\textwidth]{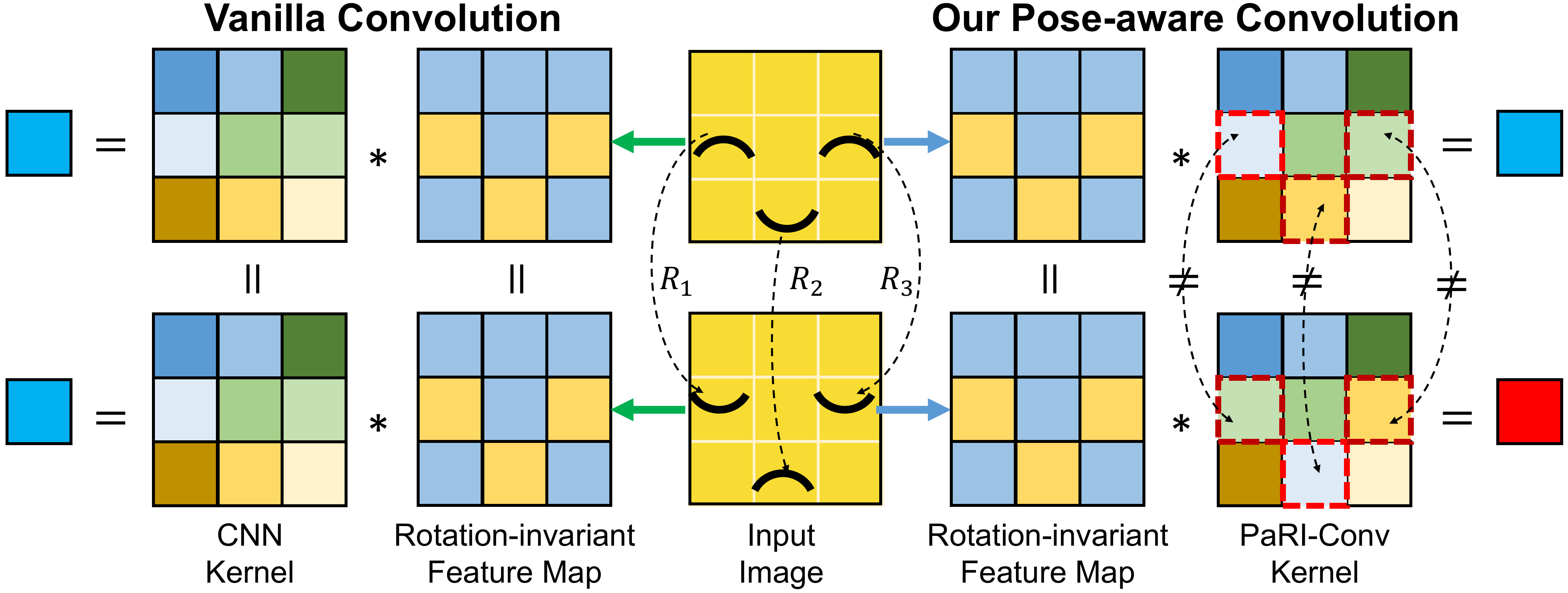}
		\vspace{-15pt}
	\caption{Illustration of the geometric ambiguity caused by the \emph{pose information loss} problem in rotation-invariant (RI) learning. RI features naturally neglect their own pose information, and remain the same under rotations $\{R_i\}_{i=1}^3$. Thus, the \textbf{vanilla convolution} (whether on images or 3D point clouds) cannot capture the relative poses between the eyes and the mouth, leading to an ambiguous representation that fails to distinguish a smile face from an angry face. Our \textbf{Pose-aware convolution} dynamically adapts the weight based on relative pose, thus eliminating the ambiguity.}\label{fig:intro}
			\vspace{-20pt}
\end{figure}

With the development of 3D scanning technology, deep learning on 3D point clouds has achieved remarkable progress in various tasks~\cite{guo2020deep}. However, most methods \cite{dgcnn, li2018pointcnn, liu2019relation} assume a strong prior, \ie, the data are pre-aligned to the same canonical pose, and the performance will degrade drastically on unaligned data, even when extensive rotation augmentation is applied. This hinders the application of current methods to real scenarios, where objects typically appear with arbitrary rotations. 

To this end, most methods design \emph{rotation-invariant} (RI) representations as input~\cite{chen2019clusternet,li2021rotation, NIPS2020KimLocal, zhang2020learning}, and have achieved consistent invariance to rotation. However, such RI features lose global positional information, comparing to 3D coordinates, leading to severe performance degradation.
Thus, state-of-the-art methods either design auxiliary blocks~\cite{zhao2019rotation,yu2020deep} that directly process coordinates, or handcraft complex representations that record pair-wise relationship in a much larger neighbourhood~\cite{li2021rotation,xu2021sgmnet}, which incur large computational cost while still being less competitive than rotation-sensitive methods~\cite{dgcnn,li2018pointcnn} on aligned data. Moreover, they fail to explain why 3D convolutional neural networks (CNNs) on point clouds~\cite{PointNet++,liu2019relation,xu2021sgmnet,thomas2019kpconv} cannot acquire intact global information from RI features via hierarchical abstraction, as 2D CNNs can from pixels, 
since both pixels and RI features are lack of positional information.

In this paper, we reveal that the above issues stem from an inherent \emph{pose information loss problem}, 
which has not been well exploited before. Intrinsically, being invariant to rotation also means the loss of pose information. Thus, as illustrated in Figure~\ref{fig:intro}, when aggregating RI features via vanilla convolutions, the relative poses between these features are inevitably lost, leading to an ambiguous representation that fails to distinguish the smile face from the angry face. 
Moreover, this also explains the ineffectiveness of current methods\cite{chen2019clusternet,zhang2019rotation,xu2021sgmnet,NIPS2020KimLocal}, \ie, they only restore 3 degree-of-free (DoF) position rather than 6 DoF pose information.

To address the above challenges, we present \underline{P}ose-\underline{a}ware \underline{R}otation \underline{I}nvariant \underline{Conv}olutions (\emph{i.e.}, PaRI-Conv) to restore the pose information that is lost by the RI features. As shown in Figure~\ref{fig:intro}, in a convolution, our key idea is to generate a dynamic kernel weight for each neighbour based on its relative pose to the center. In this way, the relative poses between local neighbours are fully preserved in the derived feature. Specifically, we first propose an \emph{Augmented Point Pair Feature} (APPF) to represent the relative pose information, which extends the point pair feature (PPF)~\cite{drost2010model}, and is rigorously rotation-invariant.
Then the APPFs are applied to dynamically adjust convolution kernels. 
To reduce the expensive computational cost and memory burden introduced by storing heavy weight banks~\cite{xu2018spidercnn,thomas2019kpconv, xu2021paconv} or regressing large kernel matrix~\cite{simonovsky2017dynamic,wu2019pointconv}, we propose a \emph{factorized dynamic kernel}, which decomposes the kernel weight into a basis matrix that is shared by all neighbours and a low DoF diagonal matrix that is learned from the APPF. 
Since PaRI-Conv fully preserves the geometric relationship among adjacent patches by incorporating pose information, global context can be automatically acquired by simply stacking PaRI-Conv layers,
thus avoiding redundant blocks~\cite{zhao2019rotation} or complex representations~\cite{li2021rotation,xu2021sgmnet} for global information compensation, resulting in a much more compact and efficient network.
Extensive experiments on shape classification and part segmentation tasks show that our method surpasses state-of-the-art RI methods, while being consistently rotation-invariant. More importantly, our method approaches or even surpasses the state-of-the-art rotation-sensitive methods on aligned data, verifying PaRI-Conv's ability in preventing information loss.

In conclusion, the main contributions of our work are:
\begin{itemize}
	\vspace{-8pt}
	\item[$\bullet$] We reveal the pose information loss problem and address it by proposing a \underline{P}ose-\underline{a}ware \underline{R}otation \underline{I}nvariant \underline{Conv}olution operator (\emph{i.e.}, PaRI-Conv), leading to a more powerful and efficient solution for RI learning.
		\vspace{-8pt}
	\item[$\bullet$] A lightweight RI feature, named Augmented Point Pair Feature (APPF), is proposed to fully encode the pose of each neighbour relative to the center.
		\vspace{-8pt}
	\item[$\bullet$] To synthesize pose-aware kernel, a factorized dynamic kernel is designed by decomposing the kernel weight into a shared basis matrix and a pose-aware diagonal matrix, which is more compact and efficient without sacrificing the flexibility.
\end{itemize}

\section{Related Work}

\noindent
\textbf{Canonical Pose Prediction Methods.}
These methods learn to transform the shape into its canonical pose to avoid the rotation perturbation~\cite{qi2017pointnet,rotpredictor, li2021closer, sun2021canonical}. PointNet~\cite{qi2017pointnet} and following works~\cite{dgcnn,li2018pointcnn} use the Spatial Transformer Networks (STN)~\cite{STN}. RotPredictor~\cite{rotpredictor} introduces self-consistency and enables approximately equivariant pose estimation. Its stability is improved in~\cite{spezialetti2020learning} by using Spherical CNNs~\cite{cohen2018spherical} as backbone. Sun \etal~\cite{sun2021canonical} further learn category-level canonical pose by decomposing each shape into several capsules. Li \etal~\cite{li2021closer} leverage the prior of principal component analysis (PCA), and tackle its ambiguity by blending 24 PCA-based poses. Generally, above methods rely on exhaustive augmentation while only achieve approximate invariance that may not generalize to novel or partial objects. 

\noindent
\textbf{Rotation Invariant Learning on 3D Point Clouds.}
Most methods achieve rotation invariance by designing handcrafted RI features based on distances and angles~\cite{zhang2019rotation, chen2019clusternet, zhao2019rotation, zhang2020learning, yu2020deep, li2021rotation, xu2021sgmnet}. 
However, they only encode local geometry and lose global positional information~\cite{zhang2019rotation}. To alleviate this, ClusterNet~\cite{chen2019clusternet} has to build a strongly connected $k$nn-graph with $k=80$. Recent works design auxiliary blocks~\cite{zhao2019rotation,yu2020deep} to encode canonicalized points for global information compensation.
Other methods~\cite{li2021rotation, xu2021sgmnet} introduce RI representations that encode more global context, such as gram matrix, which suffers from high complexity for computing relative distances or angles between arbitrary point pairs. LRF-based methods~\cite{NIPS2020KimLocal} transform local point coordinates w.r.t the LRF to stay RI, which also loses the pose information. Equivariant methods allow the learned features to undergo predictable linear transformations following the rotations of the inputs~\cite{thomas2018tensor, poulenard2021functional, Deng_2021_ICCV, fuchs2020se}, thus can also achieve invariance without losing pose information. Yet they involve strict constraints on the convolution kernels, sacrificing their flexibility. On the contrary, our PaRI-Conv can perfectly solve the pose information loss problem, with informative convolution kernels, resulting in better performance.

\noindent
\textbf{Dynamic Kernels for Convolution on 3D Point Clouds.}
A line of works~\cite{wu2019pointconv, liu2019relation, wang2019graph} directly regress weight kernels based on the positions of neighbour points, leading to prohibitively large memory. Others maintain a weight bank and generate new weight matrix via linear combination~\cite{thomas2019kpconv, mao2019interpolated, xu2021paconv, xu2018spidercnn}. They typically relate the weights with anchor points and assemble new weight in a handcrafted manner~\cite{thomas2019kpconv, mao2019interpolated}, which limits their flexibility. Recently, PAConv~\cite{xu2021paconv} learns to dynamically assemble the kernel weights via score prediction. Generally, the weight bank requires high memory and is hard to be optimized jointly. On the contrary, our factorized dynamic kernel is more compact and efficient without sacrificing the performance.

\section{Problem Definition and Background}

We first provide necessary background and mathematically explain the pose information loss problem in existing RI networks, which we aim at solving in this paper.

\noindent
\textbf{Rotation Invariant Functions.}
Given a point cloud $P=[p_1, p_2,...,p_N]\in\mathbb{R}^{N\times3}$ with $N$ points, we can apply arbitrary rotation to it by $P=PR$, where $R\in {\rm SO}(3)$ is a $3\times3$ rotation matrix. A \emph{rotation-invariant} (RI) function $\Phi$ should satisfy:
\begin{equation}\label{def}
\Phi(P)=\Phi(PR), \ \ \  \forall R\in {\rm SO}(3).
\end{equation}

\noindent
\textbf{Convolutions on 3D Point Clouds.}
Generally, the input of a convolution layer on 3D point clouds is a point cloud $P\in \mathbb{R}^{N\times 3}$ with features $X=[x_1,...,x_N]\in \mathbb{R}^{N\times c_{in}}$, and output features $X'=[x_1',...,x_N']\in \mathbb{R}^{N\times c_{out}}$. The convolution operation $f$ at the reference point $p_r$ can be formulated as:
\begin{equation}\label{equ:3dconv1}
f(p_r) = \mathop{\bigwedge}_{j\in\mathcal{N}(p_r)}W_j\cdot h(p_j), 
\end{equation}
where $\mathcal{N}(p_r)$ denotes a local patch around $p_r$, $W_j$ is the kernel weight, and $h:\mathbb{R}^3\rightarrow\mathbb{R}^c$ denotes a non-linear function that maps the point coordinate $p_j$ to its feature $x_j$. $\bigwedge$ denotes an aggregation function, such as MAX, AVG or SUM. Early \emph{multi-layer perceptron} (MLP) based methods~\cite{qi2017pointnet, PointNet++} design isotropic kernels, meaning that $W_i=W_j,\forall i,j\in\mathcal{N}(p_r)$, which limits the ability of expressing location relationship of neighbours. More recently, some works~\cite{wu2019pointconv, mao2019interpolated, zhou2021adaptive} design position-adaptive kernels to solve this problem with $W_j=W(p_j,p_r)$.

\noindent
\textbf{Pose Information Loss Problem in Current RI Methods.}
Though above convolutions have achieved impressive performance on 3D raw coordinates, we discover that applying above convolutions on RI features inevitably loses \emph{relative pose information} between them, leading to local geometric ambiguity. Here, we extend Equation~\ref{equ:3dconv1} as:
\begin{equation}\label{equ:RIambiguity}
f(p_r|\bm{\Gamma}(p_r)) = \mathop{\bigwedge}_{j\in\mathcal{N}(p_r)}W_j\cdot h(p_j|\mathbf{\Omega}(p_j)), 
\end{equation}
where $\mathbf{\Omega}(p_j)$ is the receptive field of $h$ at point $p_j$, meaning that $h(p_j)$ only depends on $\mathbf{\Omega}(p_j)$. Then, the receptive field of $f$ in Equation~\ref{equ:RIambiguity} is defined as
$\bm{\Gamma}(p_r)=\bigcup_{j\in\mathcal{N}(p_r)}\mathbf{\Omega}(p_j).$

To explain the problem, we change the relative pose between local patches via a transformation $\mathcal{T}$:
\begin{equation}\label{equ:def_transformation}
\mathcal{T}(\bm{\Gamma}(p_r))=\bigcup_{j\in\mathcal{N}(p_r)}\mathbf{\Omega}(p_j)R_j,
\end{equation}
which means we separately rotate each patch $\bm{\Omega}(p_j)$ around $p_j$ via $R_j$ (\eg, rotating the eyes and mouth in Figure~\ref{fig:intro}, changing the smile face to the angry face). Then we perform convolution on the transformed data $\mathcal{T}(\bm{\Gamma}(p_r))$ by
\begin{equation}\label{equ:RIambiguity_2}
f(p_r|\mathcal{T}(\bm{\Gamma}(p_r))) = \mathop{\bigwedge}_{j\in\mathcal{N}(p_r)}W_j\cdot h(p_j|\mathbf{\Omega}(p_j)R_j).
\end{equation}
Given that $h$ is an RI function, we have $h(p_j|\mathbf{\Omega}(p_j))=h(p_j|\mathbf{\Omega}(p_j)R_j)$, which is the definition of RI functions given in Equation~\ref{def}. Thus, the right-hand side of Equation~\ref{equ:RIambiguity} and Equation~\ref{equ:RIambiguity_2} is equal, leading to
\begin{equation}\label{equ:def_am}
f(p_r|\bm{\Gamma}(p_r))=f(p_r|\mathcal{T}(\bm{\Gamma}(p_r))),
\end{equation}
which means current CNNs cannot identify the changes of the relative pose between local patches achieved by $\mathcal{T}$, and proves the derived feature will \emph{inevitably lose the relative pose information} between local patches. This can cause severe geometric ambiguity in the derived feature, \ie, as shown in Figure~\ref{fig:intro}, the feature derived from vanilla CNNs cannot distinguish the happy face from the angry face.

\section{Method}

\begin{figure*}[htbp]
	\centering
	\includegraphics[width=0.94\textwidth]{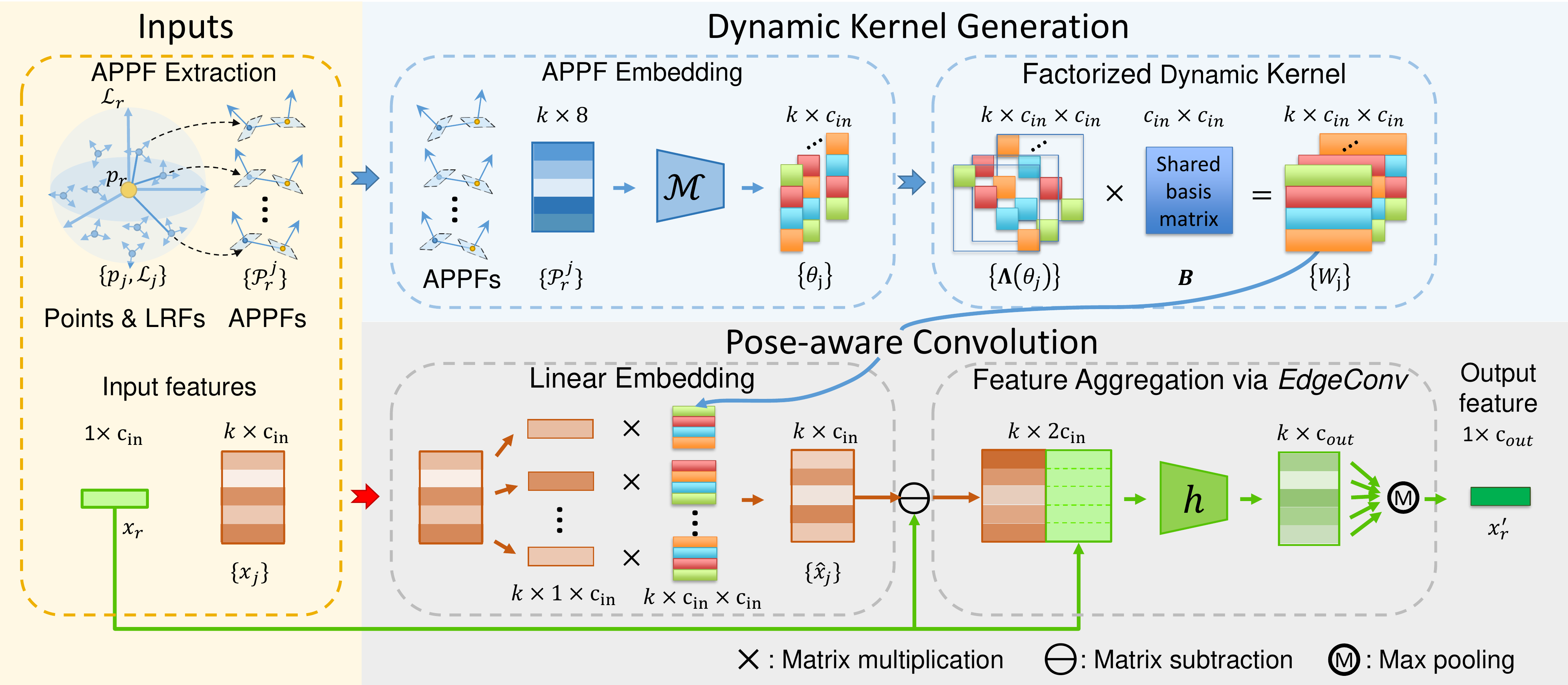}
	\caption{Overview of the proposed PaRI-Conv. The inputs of a PaRI-Conv layer are the LRF $\mathcal{L}_r$ and feature $x_r$ for a center point $p_r$, as well as the LRFs $\{\mathcal{L}_j\}_{j=1}^k$ and features $\{x_j\}_{j=1}^k$ for its $k$ neighbours $\{p_j\}_{j=1}^k$. To address the pose information loss problem, the key idea is to synthesize dynamic kernels $\{W_j\}_{j=1}^k$ based on the relative poses. We encode the relative pose between each neighbour $p_j$ and the center $p_r$ via the proposed APPF $\mathcal{P}_r^j$, which is extracted from their respective LRFs $\mathcal{L}_j$ and $\mathcal{L}_r$. Then $\mathcal{P}_r^j$ is encoded into $\theta_j$ via an MLP $\mathcal{M}$, and converted to a diagonal matrix $\bm{\Lambda}(\theta_j)$, which is multiplied with the shared basis matrix $B$ to synthesize the final kernel weight $W_j$. We name this strategy \emph{factorized dynamic kernel}. After the linear embedding, we aggregate the derived features $\{\hat{x}_j\}^k_{j=1}$ via EdgeConv.}\label{fig:main}
\end{figure*}

To counter the pose information loss problem, we propose a \emph{Pose-aware Rotation Invariant Convolution} (\ie, PaRI-Conv), which replaces the fixed kernels in vanilla convolutions with pose-aware dynamic kernels. 
Specifically, the general formulation of PaRI-Conv at a reference point $p_r$ can be defined as:
\begin{equation}
x_r' = \mathop{\bigwedge}_{j\in\mathcal{N}(p_r)}\mathcal{W}(\mathcal{P}_r^j)\cdot x_j, 
\end{equation}
where $\mathcal{W}$ is a dynamic kernel function that maps the relative pose $\mathcal{P}_r^j$ between the center point $p_r$ and its neighbour $p_j$ into respective kernel weight.
With the pose-aware kernel $\mathcal{W}(\mathcal{P}_r^j)$, $W_j$ in Equation~\ref{equ:RIambiguity} and~\ref{equ:RIambiguity_2} will no longer be equal, because rotating $p_j$ with $R_j$ changes its relative pose $\mathcal{P}_r^j$ to the center point. Thus, the ambiguity defined in Equation~\ref{equ:def_am} is readily eliminated. Meanwhile, the rotation invariance property is maintained because a global rotation will not change the relative poses $\mathcal{P}_r^j$ between arbitrary points (see supplementary material for more details).
The pipeline of our PaRI-Conv is illustrated in Figure~\ref{fig:main}. An RI feature, named \emph{Augmented Point Pair Feature} (APPF) is introduced to fully embed the pose information $\mathcal{P}_r^j$. Afterwards, we propose a \emph{factorized dynamic kernel} to formulate the kernel function $\mathcal{W}$, and the final output is aggregated from the derived features via the EdgeConv~\cite{dgcnn}.

\subsection{Augmented Point Pair Feature}\label{sec:APPF}

\noindent
\textbf{LRF Construction.}
\begin{figure}[htbp]
	\centering
	\includegraphics[width=0.4\textwidth]{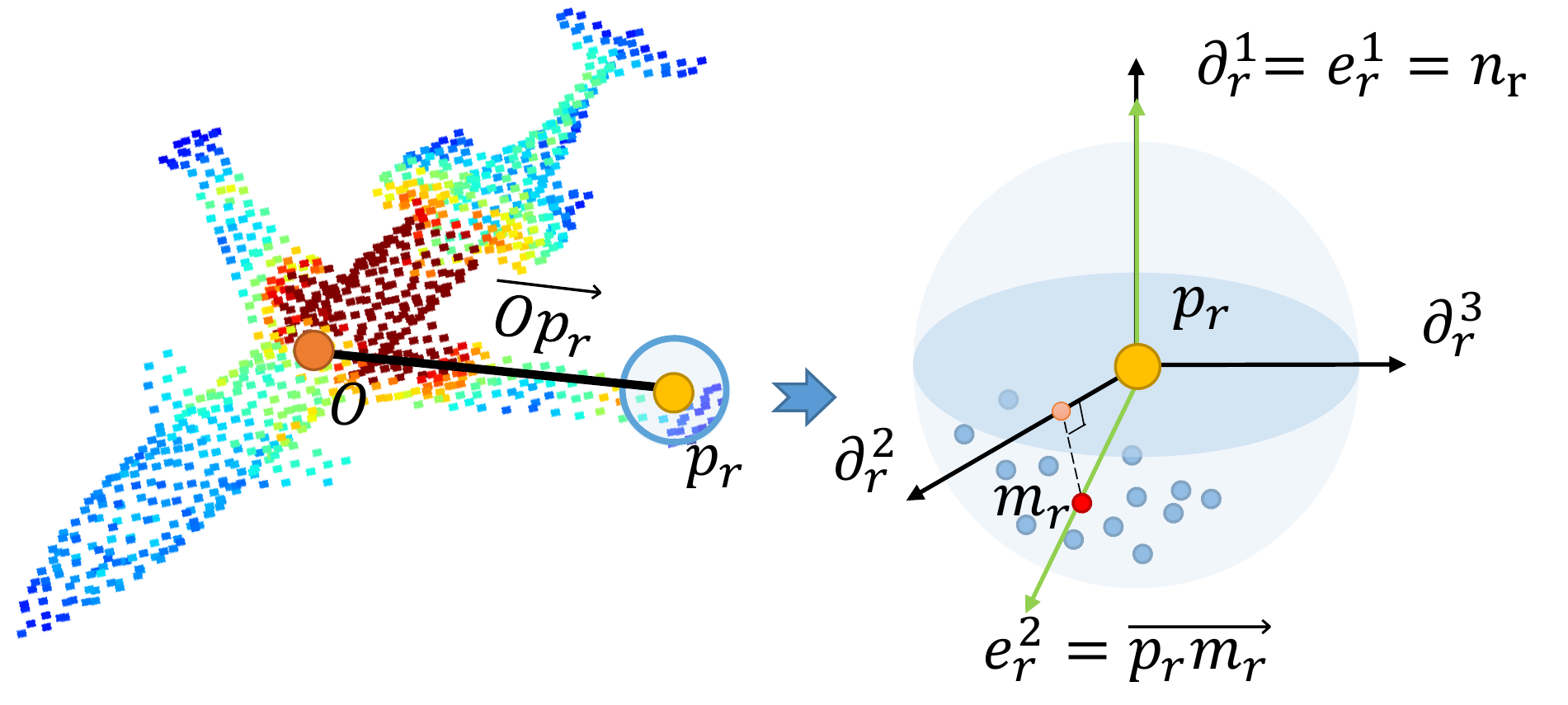}
	\caption{Illustration of the construction of LRF at point $p_r$.}\label{fig:lrf}
	\vspace{-15pt}
\end{figure}
To build RI representations, the LRF $\mathcal{L}_r$ is first constructed at each point $p_r\in P$, which consists of three orthonormal basis: $\mathcal{L}_r=[\partial_r^1, \partial_r^2, \partial_r^3]\in\mathbb{R}^{3\times 3}$. Typically, $\mathcal{L}_r$ is constructed by first defining two axes $e_r^1, e_r^2$, and then performing orthonormalization via
\begin{equation}\label{equation:lrf}
\partial_r^1=e_r^1,\ 
\partial_r^3=\frac{\partial_r^1\times e_r^2}{\|\partial_r^1\times e_r^2\|},\ 
\partial_r^2 = \partial_r^3\times\partial_r^1.
\end{equation}
There are a few choices for axes $e_r^1, e_r^2$. RI-GCN~\cite{NIPS2020KimLocal} computes LRF via PCA, which can be sensitive to perturbations. Many other methods turn to apply the vector between the global center $O$ and the local center point $p_r$~\cite{zhang2020learning, li2021rotation, yu2020deep}. Though being more consistent under rotations, the axis $\overrightarrow{Op_r}$ is related to the global location of the point $p_r$ rather than the local geometry, which leads to two serious problems: 1) the network struggles to learn consistent representations for identical structures that locate differently in the 3D shape, and 2) when partiality and background exist, the global center will be severely shifted, limiting the robustness to these perturbations. Thus, we propose to build the LRF upon local geometry only. Specifically, as shown in Figure~\ref{fig:lrf}, we define the first two axes as $e_r^1=n_r$ and $e_r^2=\overrightarrow{p_rm_r}=m_r-p_r$, where $n_r$ is the normal and $m_r=\frac{1}{k}\sum_{j=1}^{k}p_r^j$ is the barycenter of $k$ nearest neighbours around point $p_r$.

\noindent
\textbf{Feature Extraction.}
\begin{figure}[htbp]
	\centering
	\includegraphics[width=0.48\textwidth]{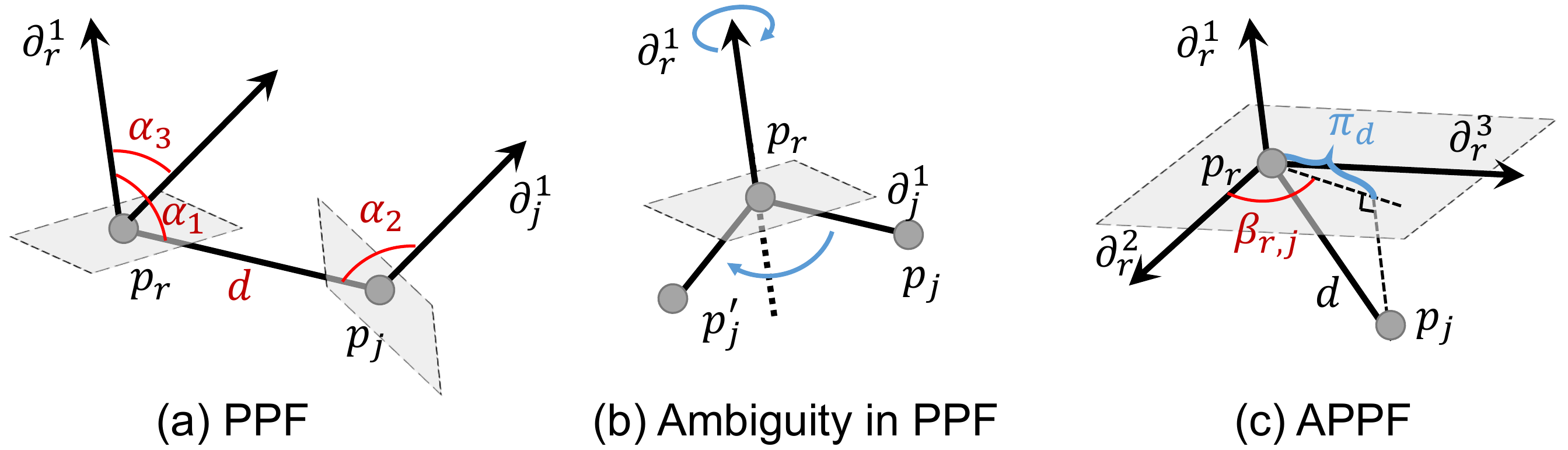}
	\caption{Illustration of the traditional PPF (a) and our proposed APPF (c). In (b), the PPF allows rotations around the axis $\partial_r^1$, leading to ambiguity. In (c), we fix the rotation via the angle $\beta_{r,j}$.}\label{fig:ppf}
	\vspace{-5pt}
\end{figure}
To capture the relative pose between the LRFs of the reference point $p_r$ and a neighbour point $p_j$, we design our RI feature based on point pair feature (PPF)~\cite{drost2010model}. As shown in Figure~\ref{fig:ppf}(a), given their primal axes $\partial_r^1, \partial_j^1$, the PPF is a 4D vector:
\begin{equation}\label{equ:ppf}
\begin{split}
{\rm PPF}(p_r,p_j)&=(\|d\|_2, \alpha_1, \alpha_2, \alpha_3),\\
\alpha_1=\angle(\partial_r^1, d),\ \alpha_2=&\angle(\partial_j^1, d),\ \alpha_3=\angle(\partial_r^1, \partial_j^1)),
\end{split}
\end{equation}
where $d=p_j-p_r$, and $\angle(\cdot,\cdot)$ denotes the angle between two vectors. However, as shown in Figure~\ref{fig:ppf}(b), PPF cannot fully define the relative pose because it allows $p_j$ to freely rotate around the axis $\partial_r^1$ of $p_r$~\cite{deng2018ppf}, and vice versa. Thus, we propose to fix the rotation by augmenting PPF with azimuth angles $\beta_{r,j}$ and $\beta_{j,r}$. As shown in Figure~\ref{fig:ppf}(c), we project $d$ on the $\partial_r^2\times \partial_r^3$ plane into $\pi_d$, and we record the azimuth angle as $\beta_{r, j}=\angle(\partial_r^2, \pi_{d})$. Similarly, $\beta_{j, r}$ can be derived by switching the roles of $p_r$ and $p_j$. To this end, our APPF $\mathcal{P}_r^j\in \mathbb{R}^8$ can be given by:
\begin{equation}
\begin{split}
\mathcal{P}_r^j=&(\|d\|_2, \cos(\alpha_1), \cos(\alpha_2), \cos(\alpha_3),\\
&\cos(\beta_{r, j}), \sin(\beta_{r, j}), \cos(\beta_{j, r}), \sin(\beta_{j, r})),
\end{split}
\end{equation}
where we further embed the angles with $\sin(\cdot)$ and $\cos(\cdot)$, following recent positional embedding techniques~\cite{mildenhall2020nerf,tancik2020fourier}. For analysis on the rotation invariance of APPF, please refer to the supplementary material. To restore the pose information lost by general convolution layers, the APPFs are then utilized to synthesize pose-aware dynamic kernel, which is introduced as follows. 

\subsection{Factorized Dynamic Kernel}

Current methods synthesize dynamic kernels by assembling memory intensive weight banks~\cite{xu2018spidercnn, mao2019interpolated, xu2021paconv} or regressing large kernel matrix~\cite{simonovsky2017dynamic,wu2019pointconv}, which incur heavy memory usage and large computational cost. To address this problem, we propose a new \emph{factorized dynamic kernel} (FDK). The key idea is to decompose the dynamic kernel $W_j=\mathcal{W}(\mathcal{P}_r^j)\in \mathbb{R}^{c_{in}\times c_{in}}$ into a lightweight pose-aware diagonal matrix $\bm{\Lambda}(\theta_{j})\in\mathbb{R}^{c_{in}\times c_{in}}$ and a pose-agnostic basis matrix $B\in\mathbb{R}^{c_{in}\times c_{in}}$ that is shared by all neighbours:
\begin{equation}\label{equ:FDK}
W_j = \bm{\Lambda}(\theta_{j}) B.
\end{equation}
As shown in Figure~\ref{fig:main}, given APPF $\mathcal{P}_r^j$ for each neighbour point, we first synthesize the diagonal matrix $\bm{\Lambda}(\theta_{j})$ by:
\begin{equation}\label{equ:dynamic_vector}
\theta_{j} = \mathcal{M}(\mathcal{P}_r^j), 
\end{equation}
where $\theta_{j}\in\mathbb{R}^{c_{in}}$ is the diagonal entries of $\bm{\Lambda}(\theta_{j})$, and $\mathcal{M}$ is a Multi-layer Perceptron (MLP). Then the pose-aware kernel $\bm{\Lambda}(\theta_{j})$ is multiplied with the shared basis matrix $B$ to generate the final kernel weight $W_j$, following Equation~\ref{equ:FDK}.

In this way, the complexity and memory usage can be significantly reduced since FDK only regresses a $c_{in}$-dim vector $\theta_j$ rather than a full $c_{in}\times c_{in}$ matrix as in~\cite{simonovsky2017dynamic,wu2019pointconv}. Moreover, FDK only requires one basis matrix $B$ shared by all neighbours. Thus, it requires much less parameters than the weight bank assembling methods~\cite{xu2018spidercnn}. We also compare their complexity in the ablation study (see Table~\ref{table:dy_ker}). 

\noindent
\textbf{Pose-aware Convolution.}
Before performing convolution with the learned pose-aware kernels $W_j$, there remains one problem. Generally, the convolution also considers the feature of the reference point $x_r$. However, in this case, the APPF $\mathcal{P}_r^r$ is a zero vector, leading to an invalid kernel matrix $W_r$. Thus, naive convolution will lose critical information from the feature of the reference point $x_r$. We address this problem by excluding $x_r$ when linear embedding:
\begin{equation}\label{lin_embed}
\hat{x}_j = W_jx_j, \quad j\neq r,
\end{equation}
and then aggregating $x_r$ with neighbour features $\hat{x}_j$ via the EdgeConv~\cite{dgcnn}:
\begin{equation}\label{aggregation}
x_r'=\mathop{\rm MAX}_{j\in\mathcal{N}(p_r)} g((\hat{x}_j-x_r)\oplus x_r),
\end{equation}
where $g$ is a one-layer MLP and $\oplus$ denotes feature concatenation operation.

\subsection{Network Architecture}
Apart from the convolution, specialized architectures can also affect the performance, and obscure the true effectiveness of the convolution operators~\cite{liu2020closer}. To fairly evaluate the intrinsic effectiveness of PaRI-Conv and its compatibility with existing backbones, we directly integrate our PaRI-Conv into two known architectures.

\noindent
\textbf{Shape Classification.} 
We employ the classical DGCNN \cite{dgcnn} by simply replacing EdgeConv in it with our PaRI-Conv. Following DGCNN, we also perform $k$ nearest neighbour ($k$nn) search ($k=20$) in Euclidean space for the first layer and in the feature space for the rest of the layers.

\noindent
\textbf{Shape Part Segmentation.}
For part segmentation, larger context is very important. Thus, we apply pooling layers to enlarge the receptive field following AdaptConv~\cite{zhou2021adaptive}. Generally, it is similar to the classification network but deeper with 5 convolution layers and 3 pooling layers. We replace all the Conv-layers with our PaRI-Conv, except for the last graph convolution layer. Here, the $k$nn search is performed in the Euclidean space with $k=40$ in all layers.

\section{Experiments}

We evaluate the performance of the proposed PaRI-Conv on three challenging datasets, \ie, ModelNet40~\cite{wu20153d} for 3D shape classification, ShapeNetPart~\cite{shapenet} for part segmentation and ScanObjectNN~\cite{uy2019revisiting} for real-world shape classification. To evaluate the invariance under various rotations, we conduct experiments under 3 train$/$test settings, \ie, $z/z$, $z/{\rm SO}(3)$ and ${\rm SO}(3)/{\rm SO}(3)$, where $z$ denotes that the inputs are processed with rotations around the vertical axis, and ${\rm SO}(3)$ denotes arbitrary rotations.

\subsection{Implementation Details}
\noindent
\textbf{Weight Kernel Generation.}
We implement the function $\mathcal{M}$ in Equation~\ref{equ:dynamic_vector} as a 2-layer MLP that ends with a linear layer. The shared basis matrix $B$ is randomly initialized and optimized together with other network parameters.

\noindent
\textbf{LRF Setting.}
The default LRF introduced in Section~\ref{sec:APPF} utilizes the normal vector. We denote the results obtained under this setting as (pc+normal) in the experiments. For fair comparison with methods that only require point positions, we replace the normal with the vector between the global center and the point (\ie,$\overrightarrow{Op_r}$), and leave the second axis unchanged. We also evaluate the performance of PaRI-Conv under more LRFs settings in the ablation study.

\noindent
\textbf{Network Inputs.}
The rotation invariance of input is an essential condition for the rotation invariance of the network. We assign each point $P_i$ with an initial RI attribute following Spherical CNN~\cite{cohen2018spherical}, which are $\|p_i\|_2$, $\sin(\angle(\partial_i^1, p_i))$ and $\cos(\angle(\partial_i^1, p_i))$, where $\partial_i^1$ is the primal axis of the LRF.

\noindent
\textbf{Training Strategy.}
We implement the above networks in PyTorch~\cite{paszke2019pytorch} and PyTorch Geometric~\cite{fey2019fast}. We use SGD with initial learning rate 0.1, which is gradually reduced to $0.001$ using cosine annealing~\cite{loshchilov2016sgdr}. The batch size is 32 and all networks are trained for 300 epochs. We apply $80\%$ dropout rate in the fully connected layers.

\subsection{3D Shape Classification}
\noindent
\textbf{Dataset.} ModelNet40~\cite{wu20153d} consists of $12,311$ CAD mesh models from 40 categories, with $9,843$ for training and $2,468$ for testing. For fair comparison, we use the sampled point clouds provided by PointNet~\cite{qi2017pointnet} and uniformly sample $1,024$ points as input. When training, we augment the input with random scaling and translation.

\noindent
\textbf{Results.} As shown in Table~\ref{table:class}, we compare our proposed PaRI-Conv with rotation sensitive, equivariant and invariant methods. 
PaRI-Conv stays consistently rotation-invariant regardless of rotation augmentation during training and achieves outstanding accuracy of \textbf{92.3\%} with normal and \textbf{91.4\%} without normal, which outperform all comparison methods on unaligned data. 
Comparing to rotation-sensitive methods, their performance degrades drastically when testing on unseen rotations under $z/{\rm SO}(3)$ setting. With augmentation (${\rm SO}(3)/{\rm SO}(3)$), there is still a noticeable performance gap ($\sim10\%$). 
Comparing to RI methods~\cite{zhang2019rotation,xu2021sgmnet,zhang2020learning,zhao2019rotation,NIPS2020KimLocal}, our method significantly improves former best performance from LGR-Net~\cite{zhao2019rotation} by 1.2\%. Specially, ours with point only can still outperform other RI methods that apply normals, which justifies the significance in countering the pose information loss suffered by these RI methods. Moreover, comparing to equivariant methods~\cite{thomas2018tensor, poulenard2021functional, Deng_2021_ICCV} that are free of the pose information loss, PaRI-Conv surpasses them even more, which can be attributed to the high flexibility of our proposed factorized dynamic kernel over their constrained kernels. 
Finally, our method is also more effective than methods that learn to transform the input point cloud globally~\cite{li2021closer, rotpredictor}, while avoiding exhaustive rotation or permutation augmentation.

\begin{table}[htbp]
	\begin{center}
		\scalebox{0.8}{
		\begin{tabular}{c|lcccc}
			\hline
			&Method & Inputs & $z/z$ & $z/{\rm SO}(3)$ & \shortstack{${\rm SO}(3)/$\\${\rm SO}(3)$} \\
			\hline
			\multirow{5}{*}{\shortstack{Rotation-\\sensitive}}
			& PointNet~\cite{qi2017pointnet} & pc & 89.2 & 16.4 & 75.5 \\
			& PointNet++~\cite{PointNet++} & pc+n & 91.8 & 18.4 & 77.4 \\
			& PointCNN~\cite{li2018pointcnn} & pc& 92.2 & 41.2 & 84.5 \\
			& DGCNN~\cite{dgcnn} & pc & 92.2 & 20.6 & 81.1 \\
			& RotPredictor~\cite{rotpredictor} & pc & 92.1 & - & 90.7\\
			\hline
			\multirow{3}{*}{\shortstack{Rotation-\\equivariant}}
			&TFN~\cite{thomas2018tensor} & pc & 88.5 & 85.3 & 87.6 \\
			&TFN-NL~\cite{poulenard2021functional} & pc & 89.7 & 89.7 & 89.7 \\
			&VN-DGCNN~\cite{Deng_2021_ICCV} & pc & 89.5 & 89.5 & 90.2 \\
			\hline
			\multirow{8}{*}{\shortstack{Rotation-\\invariant}}
			&SF-CNN~\cite{rao2019spherical} & pc+n & 92.3 & 85.3 & 91.0 \\
			&RI-CNN~\cite{zhang2019rotation} & pc & 86.5 & 86.4 & 86.4 \\
			&Li~\etal~\cite{li2021rotation} & pc & 89.4 & 89.3 & 89.4 \\
			&SGMNet~\cite{xu2021sgmnet} & pc & 90.0 & 90.0 & 90.0\\
			&AECNN~\cite{zhang2020learning} & pc& 91.0 & 91.0 & 91.0 \\
			&LGR-Net~\cite{zhao2019rotation} & pc+n & 90.9 & 90.9 & 91.1\\
			&RI-GCN~\cite{NIPS2020KimLocal} & pc+n & 91.0 & 91.0 & 91.0\\
			&Li~\etal~\cite{li2021closer} & pc & 90.2 & 90.2 & 90.2\\
			\hline
			&Ours & pc 		& \textbf{91.4}	& \textbf{91.4} & \textbf{91.4} \\
			&Ours & pc+n& \textbf{92.4} & \textbf{92.4} & \textbf{92.3} \\
			\hline
		\end{tabular}
	}
	\end{center}
	\vspace{-10pt}
	\caption{Shape classification accuracy ($\%$) on ModelNet40 dataset under three train/test settings. `pc' and `n' stands for 3d coordinates and normals of the input point cloud, respectively.}
	\label{table:class}
	\vspace{-10pt}
\end{table}

\subsection{Shape Part Segmentation}

\noindent
\textbf{Dataset.}
For shape part segmentation, we evaluate our method on ShapeNetPart Dataset~\cite{shapenet}. It contains $16,881$ 3D shapes from 16 categories. In each category, shapes are labeled with $2\sim 5$ parts, resulting in 50 parts in total. We follow the commonly applied train-test split in~\cite{PointNet++} and randomly sample $2,048$ points from each shape as input.

\noindent
\textbf{Results.}
We utilize mean Intersection-over-Union (mIoU) over all instances as evaluation metric. As shown in Table~\ref{table:partseg}, similar conclusion can be drawn. Our method outperforms all comparison methods and shows consistent performance under unseen rotations. Moreover, ours without normal also surpasses methods with normal, \eg, LGR-Net~\cite{zhao2019rotation} and RI-GCN~\cite{NIPS2020KimLocal}, showing the significance of preserving pose information. We also visualize segmentation results under $z/{\rm SO}(3)$ setting in Figure~\ref{fig:partseg}, where most parts are well segmented under unseen rotations.

\begin{table}[htbp]
	\begin{center}
								\setlength{\tabcolsep}{5mm}
	\scalebox{0.8}{
		\begin{tabular}{lcc}
			\hline
			Method & $z/{\rm SO}(3)$ & ${\rm SO}(3)/{\rm SO}(3)$ \\
			\hline
			\multicolumn{3}{c}{Input pc only} \\
			\hline
			PointNet~\cite{qi2017pointnet} &37.8 & 74.4 \\
			PointNet++~\cite{PointNet++} & 48.2 & 76.7 \\
			DGCNN~\cite{dgcnn} & 37.4 & 73.3 \\
			RSCNN~\cite{liu2019relation} & 50.7 & 73.3\\
			\hline
			RI-CNN~\cite{zhang2019rotation} & 75.3 & 75.3\\
			VN-DGCNN~\cite{Deng_2021_ICCV} & 81.8 & 81.8 \\
			RI-Framework~\cite{li2021rotation} & 82.2 & 82.5\\
			Li \etal~\cite{li2021closer} & 81.7 & 81.7\\
			\textbf{Ours (pc)} & \textbf{83.8} & \textbf{83.8} \\
			\hline
			\multicolumn{3}{c}{Input pc+normal} \\
			\hline
			RI-GCN~\cite{NIPS2020KimLocal} & 77.2 & 77.3\\
			LGR-Net~\cite{zhao2019rotation} & - & 82.8\\

			\textbf{Ours (pc+normal)}  & \textbf{84.6} & \textbf{84.6} \\
			\hline
		\end{tabular}
	}
	\end{center}
				\vspace{-10pt}
	\caption{Part segmentation results on ShapeNetPart dataset~\cite{shapenet}. We report mIoU ($\%$) over all instances under two train/test settings. 'pc' stands for 3d coordinates of the input point cloud.}
	\label{table:partseg}
				\vspace{-10pt}
\end{table}

\begin{figure}[htbp]
	\centering
	\includegraphics[width=0.48\textwidth]{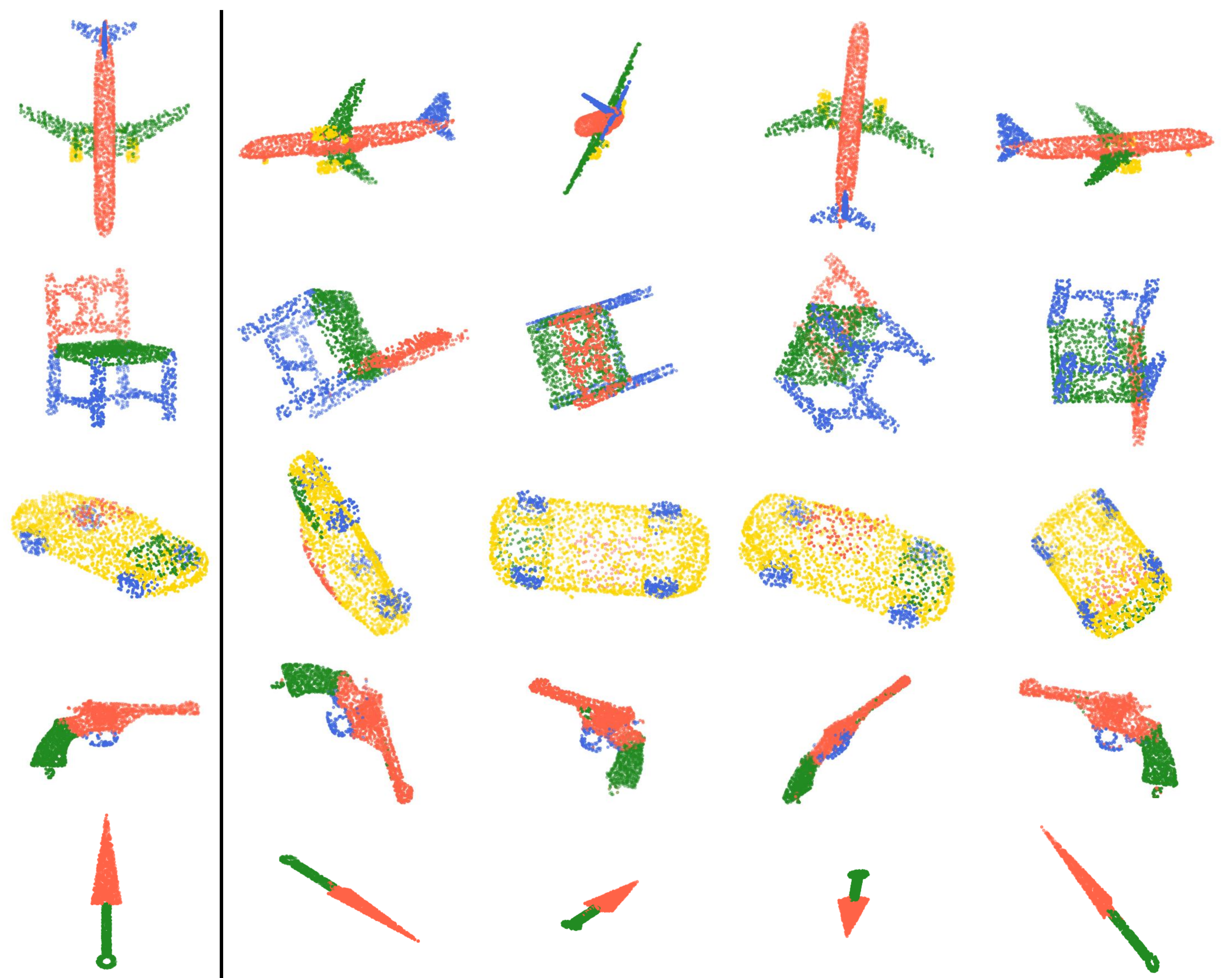}
		\vspace{-10pt}
	\caption{Visualization of part segmentation results on ShapeNetPart Dataset~\cite{shapenet} under $z/{\rm SO}(3)$ setting. The ground truth is in the left-most column. The rest columns are our testing results under shown rotations.}\label{fig:partseg}
	\vspace{-10pt}
\end{figure}

\subsection{Real-world Classification}

\noindent
\textbf{Dataset.} To test the robustness of our RI method for real-world applications, we evaluate the proposed PaRI-Conv on scanned indoor objects from ScanObjectNN~\cite{uy2019revisiting}. We choose the $\mathit{OBJ\_BG}$ subset, which contains $2,902$ objects from 15 categories. In addition to common nuisances, such as noise, incompleteness and deformations, data in this subset are also cluttered with background, which poses great challenges to the stability of rotation invariance. Note that normal is not applied since it is not available on their processed dataset. 

\noindent
\textbf{Results.} As shown in Table~\ref{table:scan}, our method can also achieve consistent invariance to rotations even when various real-world nuisances exist. Since the axis $\overrightarrow{Op_r}$ applied by \textbf{ours (pc)} can be very unstable due to the shift of global center $O$ when cluttered background exists. We replace $\overrightarrow{Op_r}$ with a PCA-based primal axis, denoted as \textbf{ours (PCA)}. Though ours (pc) is less competitive, ours (PCA) achieves the best performance, and outperforms the strongest competitor LGR-Net by 2.1\%. More importantly, when comparing the results on aligned data (\ie, $z/z$), ours (PCA) even surpasses most standard 3D deep learning methods (\eg, DGCNN~\cite{dgcnn}), which demonstrates that explicitly modeling the relative poses via our PaRI-Conv can not only eliminate the pose information loss but also counter various perturbations existing in real-world data.

\begin{table}[htbp]
	\begin{center}
						\setlength{\tabcolsep}{4.5mm}
	\scalebox{0.8}{
		\begin{tabular}{lccc}
			\hline
			Method & $z/z$ & $z/{\rm SO}(3)$ & ${\rm SO}(3)/$${\rm SO}(3)$ \\
			\hline
			\multicolumn{4}{c}{Input pc only} \\
			\hline
			PointNet~\cite{qi2017pointnet} & 73.3 & 16.7 & 54.7 \\
			PointNet++~\cite{PointNet++} & 82.3 & 15.0 & 47.4 \\
			PointCNN~\cite{li2018pointcnn} & 86.1 & 14.6 & 63.7 \\
			DGCNN~\cite{dgcnn} & 82.8 & 17.7 & 71.8 \\
			\hline
			RI-CNN~\cite{zhang2019rotation} & 74.6 & 75.3 & 75.5\\
			RI-Framework~\cite{li2021rotation} & - & 79.8 & -\\
			LGR-Net~\cite{zhao2019rotation} & 81.2 & 81.2 & 81.4\\
			\hline
			\textbf{Ours (pc)}  & 77.8 & 77.8 & 78.1 \\
			\textbf{Ours (PCA)}  & \textbf{83.3} & \textbf{83.3} & \textbf{83.3} \\
			\hline
		\end{tabular}
	}
	\end{center}
	\vspace{-10pt}
	\caption{Real-world classification accuracy ($\%$) on ScanObjectNN dataset under three train/test settings. 'PCA' denotes the primal axis $\partial_r^1$ is constructed via PCA in local neighbours.}
	\label{table:scan}
	\vspace{-10pt}
\end{table}

\subsection{Model Complexity}

Efficiency and lightness are also our key advantages. Former methods suffer from large computational cost caused by larger receptive fields~\cite{zhang2020learning} and additional network~\cite{zhao2019rotation} to incorporate the lost global information. In contrast, due to the pose-awareness, our method can obtain intact global context by simply stacking PaRI-Conv layers as in normal 2D CNNs. Moreover, PaRI-Conv is also free of the pose predictor utilized in global methods~\cite{li2021closer}. Consequently, as shown in table~\ref{table:model_complex}, our method achieves the best performance with the smallest model size and much less computational cost ($>$\textbf{48\%} FLOPs reduction).

\begin{table}[htbp]
	\centering
	\setlength{\tabcolsep}{1.2mm}
	\scalebox{0.8}{
		\begin{tabular}{l|c|cccc}
			\hline
				  & DGCNN & Li \etal~\cite{li2021closer} & AECNN~\cite{zhang2020learning} & LGRNet~\cite{zhao2019rotation} & Ours\\
			\hline
			FLOPs & 2495M & 3747M & 4841M & 5828M & \textbf{1938M} \\
			Params & \textbf{1.81M} & 2.91M & 1.99M & 5.55M & 1.85M \\
			Acc.  & 75.5 & 90.2 & 91.0 & 91.1 & \textbf{91.4} \\
			\hline
	\end{tabular}}
	\vspace{-5pt}
	\caption{Comparison of model complexity on ModelNet40 dataset, where floating point operation/sample (FLOPs), number of parameters (Params) and accuracy (Acc.) are reported.}
	\label{table:model_complex}
	\vspace{-10pt}
\end{table}

\subsection{Ablation Studies}

To demonstrate the effectiveness of each proposed component in PaRI-Conv, we further conduct ablation studies on the relative pose representation (APPF), the factorized dynamic kernel and the LRF construction. All experiments are evaluated by accuracy ($\%$) on ModelNet40~\cite{wu20153d} dataset.

\noindent
\textbf{Relative Pose Representation.} 
The pose is the key to preventing information loss. Thus, how to represent it is worth exploring.
As shown in Table~\ref{table:RIrep}, we compare our proposed APPF with: \textbf{1)} the rotation matrix and translation vector $(R,t)$, \textbf{2)} the PPF~\cite{drost2010model} defined as Equation~\ref{equ:ppf}, and \textbf{3)} APPF without direction information, which is defined as $(\|d\|_2, \cos(\alpha_1), \cos(\beta_{r,j}), \sin(\beta_{r,j}))$, where $\|d\|_2, \alpha_1, \beta_{r,j}$ can be considered as the 3-dim polar coordinates, which only encode the relative position information. 
First of all, APPF notably outperforms representations that only partially encode the pose information ($\textbf{2,3}$), demonstrating the critical role of pose-awareness in preventing information loss that exists in RI learning. Moreover, APPF also surpasses $(R,t)$ by $1.1\%$, which shows APPF is a more effective representation for deep networks to embed.

\begin{table}[htbp]
		\centering
				\setlength{\tabcolsep}{4mm}
		\scalebox{0.8}{
		\begin{tabular}{c|l|cc}
			\hline
			No.&RI Representation & Dim & ${\rm SO}(3)/{\rm SO}(3)$\\
			\hline
			1&$(R,t)$ & 12  & 91.2 \\
			2&PPF & 4 & 91.5 \\
			3&APPF-w/oDirection  & 5 & 91.4 \\
			4&APPF & 8  & \textbf{92.3} \\
			\hline
		\end{tabular}}
			\vspace{-5pt}
		\caption{Ablation study on the relative pose representation $\mathcal{P}_r^j$.}
		\label{table:RIrep}
			\vspace{-10pt}
\end{table}

\noindent
\textbf{Factorized Dynamic Kernel.}
As shown in Table~\ref{table:dy_ker}, we compare the proposed factorized dynamic kernel (FDK) with other strategies that allow pose-aware convolution, including: \textbf{1)} Concat where we ablate the FDK and directly concatenate the APPFs with the latent features, \textbf{2)} Ours-w/oEdge where we ablate the final \emph{EdgeConv}, and \textbf{3)} other dynamic kernels \cite{liu2019relation, wu2019pointconv, xu2021paconv}. With the proposed FDK, ours-w/oEdge reduces the number of parameters (up to \textbf{26}\%) and FLOPs ($>$\textbf{12}\%), while achieving better performance. This shows that FDK is more compact and flexible. Our full PaRI-Conv can further boost the performance to \textbf{92.3}\% with moderate parameters and FLOPs increment.

\begin{table}[htbp]
	\centering
		\setlength{\tabcolsep}{4.5mm}
	\scalebox{0.8}{
		\begin{tabular}{l|ccc}
			\hline
			Dynamic Kernels & Params & FLOPs & ${\rm SO}(3)/{\rm SO}(3)$\\
			\hline
			Concat. 							& 1.84M & - 	& 91.3 \\
			RSCNN~\cite{liu2019relation}  		& 1.83M & 1913M & 91.2 \\
			PointConv~\cite{wu2019pointconv} 	& 2.45M & 2264M & 91.2 \\
			PAConv~\cite{xu2021paconv}			& 2.44M	& 1768M & 89.6\\
			Ours-w/oEdge  						& \textbf{1.81M}	& \textbf{1579M}	& 91.7 \\
			Ours 								& 1.85M & 1938M & \textbf{92.3} \\
			\hline
	\end{tabular}}
	\vspace{-5pt}
	\caption{Ablation study on the factorized dynamic kernel. 
	}\label{table:dy_ker}
	\vspace{-10pt}
\end{table}

\noindent
\textbf{LRF Construction.} The selection of LRF is crucial for the robustness of rotation invariance. Recall that we introduce three relatively stable axes in Section~\ref{sec:APPF}, which are the normal $n_r$, the direction of the local barycenter $\overrightarrow{p_rm_r}=m_r-p_r$ and the direction relative to the global center $\overrightarrow{Op_r}=p_r-O$. Based on these axes, we construct various LRFs by assigning two of them as $(e_r^1,e_r^2)$ and derive the LRF following Equation~\ref{equation:lrf}.
As shown in Table~\ref{table:lrf}, the LRFs with normal generally outperform the counterparts. Moreover, replacing the local barycenter in $(n_r, \overrightarrow{p_rm_r})$ with the global center in $(n_r, \overrightarrow{Op_r})$ results in severe performance degradation ($0.6\%$), which proves our assumption that LRF should be built upon local geometry only.

\begin{table}[htbp]
	\centering
	\scalebox{0.8}{
		\begin{tabular}{l|cccc}
			\hline
			$(e_r^1, e_r^2)$ &$(\overrightarrow{p_r m_r}, \overrightarrow{Op_r})$&$(\overrightarrow{Op_r}$, $\overrightarrow{p_rm_r})$&
			$(n_r, \overrightarrow{Op_r})$&$(n_r, \overrightarrow{p_rm_r})$  \\
			\hline
			 Acc. & 91.2 & 91.4 & 91.7 & \textbf{92.3} \\
			\hline
	\end{tabular}}
	\vspace{-5pt}
	\caption{Ablation study on LRF construction.}
	\label{table:lrf}
	\vspace{-10pt}
\end{table}

\subsection{Visualization of Learned RI Features}\label{sec:vis}

We visualize the learned feature map under various rotations in Figure~\ref{fig:rot_feat_vis}. While the features learned by DGCNN \cite{dgcnn} change vastly as the shape rotates, our PaRI-Conv can learn consistent representation for corresponding points between point clouds under different rotations. 

The above results demonstrate that developing deep networks like PaRI-Conv that respect natural symmetries in 3D space, such as rotations and reflections, allows identical geometric structures to share the same convolution kernels. On the contrary, rotation-sensitive methods~\cite{dgcnn} are forced to learn redundant filters for the same structures under different poses, which indicates PaRI-Conv has the potential in developing more compact networks.
\begin{figure}[htbp]
	\centering
	\includegraphics[width=0.48\textwidth]{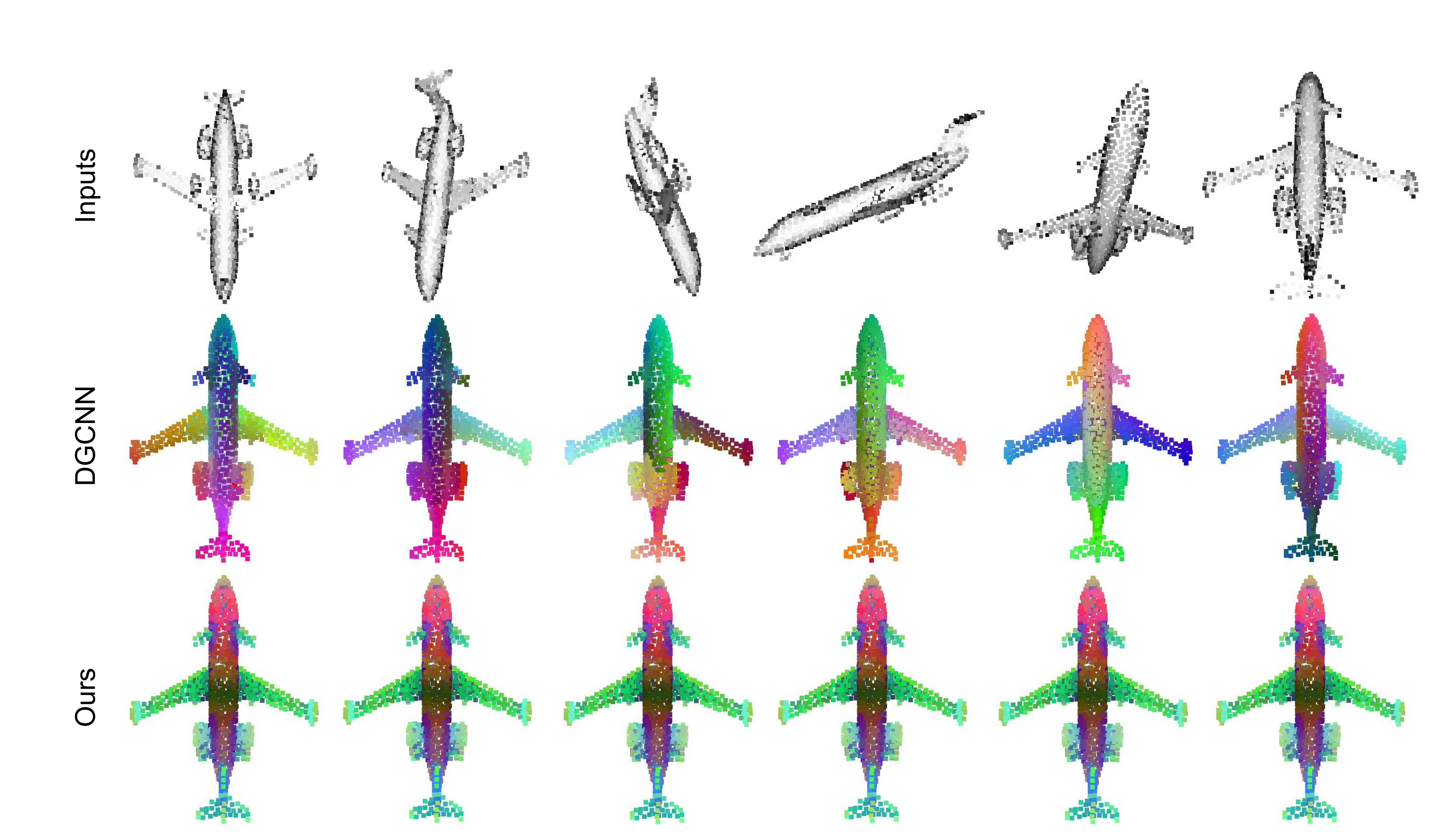}
	\caption{Visualization of dense features learned on shape classification task. While the features learned by DGCNN (second row) change vastly, our PaRI-Conv (third row) learns consistent representation for identical structures under different rotations.}\label{fig:rot_feat_vis}
	\vspace{-12pt}
\end{figure}

\subsection{Discussion}

Though our PaRI-Conv solves the information loss problem in RI learning, there remains small gap between PaRI-Conv and more recent rotation-sensitive methods~\cite{xu2021paconv,zhao2021point} on aligned data. We assume that this may be partially due to some prior that exists in aligned data. For instance, shape of beds and wardrobes are very similar, while the canonical poses (\ie, lying down vs. standing) provide strong prior, making it easier to distinguish between them, which, however, cannot be leveraged by our PaRI-Conv. To preclude this factor, we further evaluate our PaRI-Conv on aligned data by directly taking coordinates as input. Surprisingly, with normal as a stable axis, PaRI-Conv achieves $93.8\%$ overall accuracy, which outperforms recent point convolution method PAConv~\cite{xu2021paconv} ($93.6\%$) and powerful transformer based method~\cite{zhao2021point} ($93.7\%$) (see the supplementary material for details). We believe that above results reveal the significance of pose information, which has not been well exploited before, and our proposed PaRI-Conv is an effective operator in capturing pose-variant geometric structures.

\section{Conclusion, Limitation and Future Work}

We have presented PaRI-Conv, a Pose-aware convolution operator for RI learning on 3D point clouds. We have revealed that the inferior performance of current RI methods is caused by the inherent pose information loss problem, and demonstrated that PaRI-Conv can perfectly solve the problem by dynamically adapting the kernel based on the relative pose. Specifically, a new Augmented Point Pair Feature (APPF) has been proposed to effectively encode the pose. To synthesize the dynamic kernel, we factorize it into a shared basis and a low DoF dynamic kernel, which has been shown to be more lightweight without sacrificing flexibility. Experimental results show that PaRI-Conv is consistently invariant, and, by eliminating the pose information loss, can noticeably surpass the state-of-the-art RI methods with less parameters and computational cost.

Though achieving certain robustness to real-world nuisance, our method could be sensitive to extreme noises due to the instability of LRFs. Since PaRI-Conv provides a general framework for RI learning, this can be addressed by applying more robust LRFs~\cite{petrelli2012repeatable,salti2014shot}. 

In the future, we hope PaRI-Conv can pave the way for introducing rotation invariance to more tasks, \eg, semantic segmentation, object detection and registration \etc.

\appendixpageoff
\appendixtitleoff
\renewcommand{\appendixtocname}{Supplementary material}
\begin{appendices}

\include{supp}

\end{appendices}
{\small
	\bibliographystyle{ieee_fullname}
	\bibliography{main}
}

\end{document}

%% file: supp.tex
\title{\emph{Supplementary Material for} \\The Devil is in the Pose: Ambiguity-free 3D Rotation-invariant \\ Learning via Pose-aware Convolution}
\emptythanks
\author{}
\date{}

\maketitle

\thispagestyle{empty}
\setcounter{section}{0}
\setcounter{figure}{0}
\setcounter{table}{0}
\renewcommand\thesection{\Alph{section}}

In this supplementary material, we first provide theoretical analysis on the rotation invariance of our proposed PaRI-Conv in Section~\ref{sec:theo}. We further show more visualization on the learned feature map in Section~\ref{sec:vis}. Section~\ref{sec:robust} investigates the robustness of PaRI-Conv. Finally, in Section~\ref{sec:discuss}, we provide detailed comparison results for the discussion section in the main paper.

\section{Theoretical Analysis on Rotation Invariance}\label{sec:theo}
Here, we provide theoretical analysis on the rotation invariance of our proposed PaRI-Conv.

\subsection{Augmented Point Pair Feature}

We first introduce some lemmas and prove that our proposed Augmented Point Pair Feature (APPF) is rotation-invariant, which is the building block of the rotation invariance of our PaRI-Conv.
\begin{lemma}\label{lemma:angle}
Given two vector $v_1,v_2\in\mathbb{R}^{1\times 3}$, the angle between them is invariant to arbitrary rotations $R\in {\rm SO}(3)$:
\begin{equation}\label{equ:angle}
\angle(v_1, v_2)=\angle(v_1R, v_2R).
\end{equation}
\begin{proof}
Given that $\angle(v_1, v_2)\in [0,\pi]$, Equation~\ref{equ:angle} is equivalent to $\cos(\angle(v_1, v_2))=\cos(\angle(v_1R, v_2R))$, which is given by:
\begin{equation}
\begin{split}
\cos(\angle(v_1R, v_2R))&=\frac{\langle v_1R,v_2R\rangle}{\|v_1R\|\|v_2R\|}=\frac{v_1RR^\top v_2^\top }{\|v_1R\|\|v_2R\|}\\
=\frac{v_1v_2^\top}{\|v_1\|\|v_2\|}=&\frac{\langle v_1,v_2\rangle}{\|v_1\|\|v_2\|}=\cos(\angle(v_1, v_2)),
\end{split}
\end{equation}
where $\langle\cdot, \cdot\rangle$ denotes inner product.
\end{proof}
\end{lemma}

\begin{lemma}\label{lemma:dist}
	Given two points $p_1,p_2\in\mathbb{R}^{1\times 3}$, the distance between them is invariant to arbitrary rotations $R\in {\rm SO}(3)$:
	\begin{equation}\label{equ:dist}
	{\rm dist}(p_1, p_2)={\rm dist}(p_1R, p_2R).
	\end{equation}
	\begin{proof}
		\begin{equation}
		\begin{split}
		{\rm dist}(p_1R, p_2R)=&\sqrt{(p_1-p_2)RR^\top(p_1-p_2)^\top} \\
		=&\sqrt{(p_1-p_2)(p_1-p_2)^\top} \\
		=&{\rm dist}(p_1, p_2)
		\end{split}
		\end{equation}
	\end{proof}
\end{lemma}

\begin{figure}[htbp]
	\centering
	\includegraphics[width=0.44\textwidth]{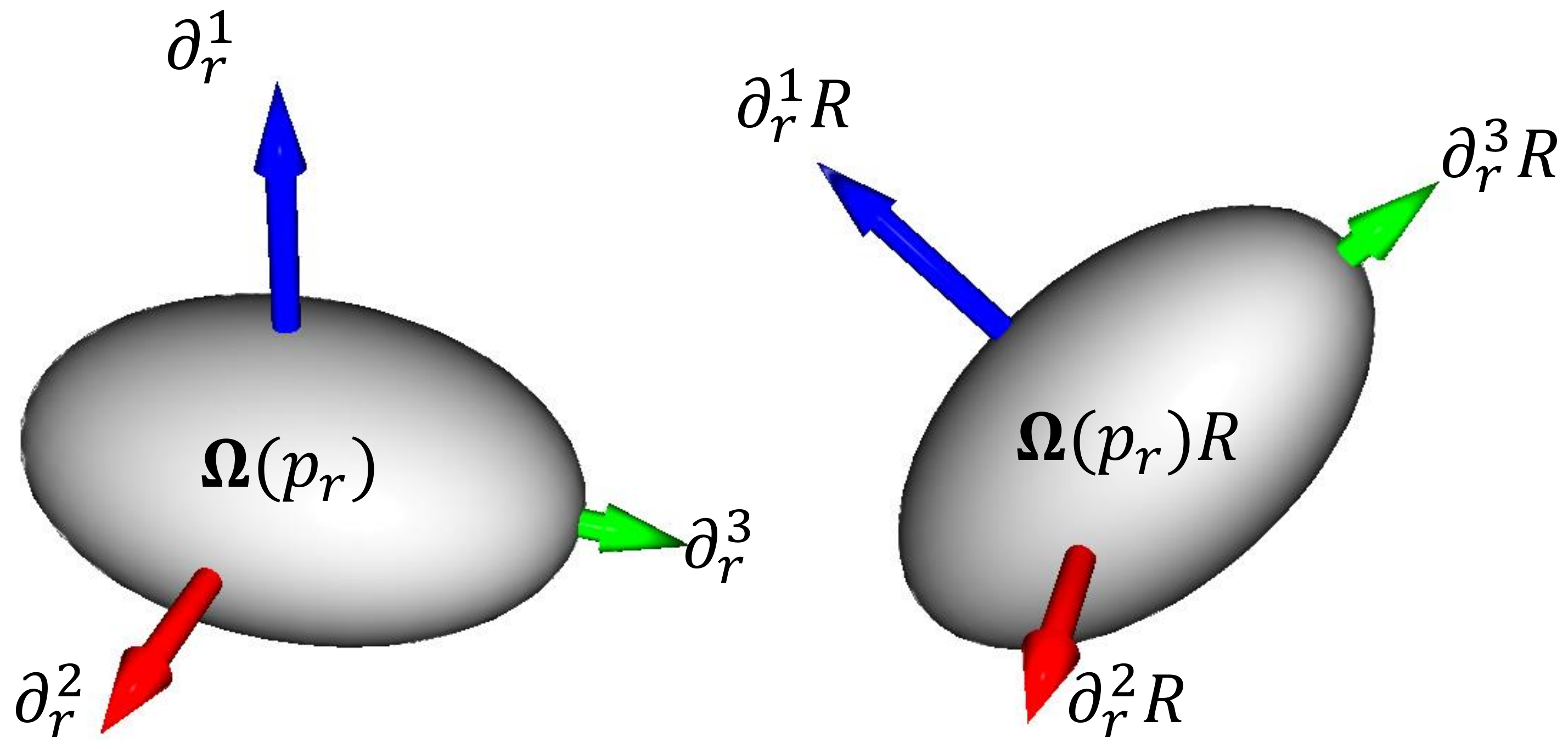}
	\caption{Illustration of the rotation equivariance of LRF.}\label{fig:lrf_prop}
\end{figure}

The rotation invariance of our APPF relies on the stability of LRF, which is expected to satisfy
\begin{equation}\label{equ:lrf}
	\partial_r^*(\bm{\Omega}(p_r)R)=\partial_r^*(\bm{\Omega}(p_r))R,
\end{equation}
where $\bm{\Omega}(p_r)$ denotes a local patch around point $p_r$. This implies an axis of the LRF $\partial_r^*$ built at $\bm{\Omega}(p_r)$ is equivariant to arbitrary rotations $\forall R\in {\rm SO}(3)$. As shown in Figure~\ref{fig:lrf_prop}, intuitively, this means a stable LRF should rotate together with the rotation of local patch $\bm{\Omega}(p_r)$.

\begin{theorem}
The APPF defined as follows is rotation-invariant.
\begin{equation}\label{equ:def_APPF}
\begin{split}
\mathcal{P}_r^j=&(\|d\|_2, \cos(\alpha_1), \cos(\alpha_2), \cos(\alpha_3),\\
&\cos(\beta_{r, j}), \sin(\beta_{r, j}), \cos(\beta_{j, r}), \sin(\beta_{j, r})),
\end{split}
\end{equation}
\end{theorem}

\begin{proof}
The first four elements $\|d\|_2, \alpha_1, \alpha_2, \alpha_3$ are from the original Point Pair Feature (PPF), whose rotation invariance has been proved in~\cite{deng2018ppf}. Thus, here we only need to prove that $\beta_{r, j}, \beta_{j, r}$ are rotation invariant. Without loss of generality, we only prove the rotation invariance of $\beta_{r, j}$. And the rotation invariance of $\beta_{j, r}$ can be derived similarly. 

As shown in Figure 4(c) of the main paper, the definition of $\beta_{r, j}$ is given by $\beta_{r, j}=\angle(\partial_r^2, \pi_{d})$. We define an induced transformation $L_R$, that acts on the functions $f$ of the point cloud $P\in \mathbb{R}^{N\times 3}$ as
\begin{equation}\label{equ:def_LR}
[L_R\circ f](P)=f(PR),
\end{equation}
where $f$ can be vectors or angles in the point cloud $P$ such as $\pi_d,\beta_{r, j}$. Intuitively, $L_R$ maps a feature $f$ to itself after being rotated by $R$. Considering the property given in Equation~\ref{equ:lrf}, $L_R\circ \partial_r^2$ can be derived by:
\begin{equation}\label{equ:trans_lrf}
L_R\circ \partial_r^2=\partial_r^2(PR)=
\partial_r^2(\bm{\Omega}(p_r)R)=\partial_r^2(\bm{\Omega}(p_r))R=\partial_r^2R.
\end{equation}
Thus, we can proceed to derive:
\begin{equation}
\begin{split}
L_R\circ\beta_{r, j}=&\angle(L_R\circ\partial_r^2, L_R\circ\pi_{d}) \\
=&\angle(L_R\circ\partial_r^2, L_R\circ(d-\langle d,\partial_r^1 \rangle\partial_r^1)) \\
=&\angle(\partial_r^2R, dR-\langle dR,\partial_r^1R\rangle\partial_r^1R) \\
=&\angle(\partial_r^2R, (d-\langle d,\partial_r^1\rangle\partial_r^1)R) \\
=&\angle(\partial_r^2R, \pi_{d}R) \\
=&\angle(\partial_r^2, \pi_{d}) \\
=&\beta_{r, j},
\end{split}
\end{equation}
which proves $\beta_{r, j}$ is rotation invariant. Detailed explanation of each procedure is given as follows. Line 2 is the extension of the projection $\pi_d$. Line 3 is given by Equation~\ref{equ:trans_lrf} and the definition of $L_R$ given in Equation~\ref{equ:def_LR}. Line 4 is simply derived by algebraic operations. Line 5 is derived by the definition of $\pi_d$. Line 6 utilizes Lemma~\ref{lemma:angle}.

\end{proof}

\subsection{PaRI-Conv}
\begin{lemma}\label{lemma:input}
The network inputs defined as $\|p_i\|_2$, $\sin(\angle(\partial_i^1, p_i))$, $\cos(\angle(\partial_i^1, p_i))$ is rotation invariant.
\end{lemma}
\begin{sproof}
The above representation is composed of distances and angles, which are proved to be rotation invariant in Lemma~\ref{lemma:angle} and Lemma~\ref{lemma:dist}. Consequently, the representation is also rotation invariant.
\end{sproof}

\begin{theorem}
The $l$-th PaRI-Conv layer defined as follows is rotation invariant.
\begin{equation}\label{equ:pari-conv}
x_r^{l+1} = \mathop{\bigwedge}_{j\in\mathcal{N}(p_r)}\mathcal{W}(\mathcal{P}_r^j)\cdot x_j^l, 
\end{equation}
\end{theorem}

\begin{proof}
We first assume the statement is true for layer $k-1$, which implies that the output features $x^k$ is rotation invariant. Then, given that $\mathcal{P}_r^j$ is also rotation invariant, we can proceed to derive:
\begin{equation}
\begin{split}
L_R\circ x_r^{k+1}=&L_R\circ \left(\mathop{\bigwedge}_{j\in\mathcal{N}(p_r)}\mathcal{W}(\mathcal{P}_r^j)\cdot x_j^k\right) \\
=&\mathop{\bigwedge}_{j\in\mathcal{N}(p_r)}\mathcal{W}(L_R\circ \mathcal{P}_r^j)\cdot \left(L_R\circ x_j^{k}\right) \\
=&\mathop{\bigwedge}_{j\in\mathcal{N}(p_r)}\mathcal{W}(\mathcal{P}_r^j)\cdot x_j^{k} \\
=&x_r^{k+1},
\end{split}
\end{equation}
which proves the statement is also true for layer $k$. 

Moreover, given that the input feature $x^0_j$ of layer $l=0$ is rotation invariant (Lemma~\ref{lemma:input}), the statement is true for layer 0. Thus, by mathematical induction, the statement is true for all $l$.
\end{proof}

\section{Visualization}\label{sec:vis}
\begin{figure}[htbp]
	\centering
	\includegraphics[width=0.48\textwidth]{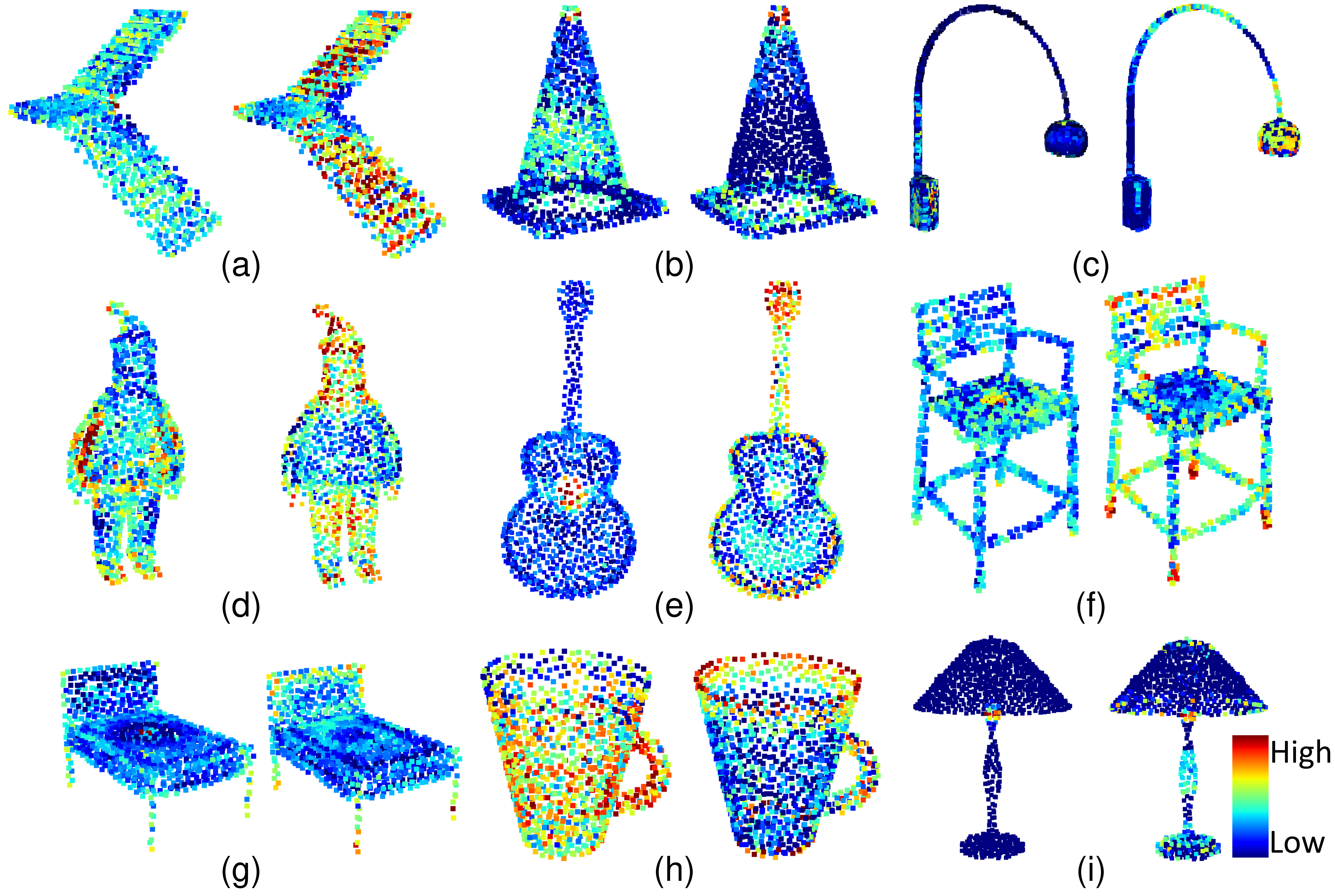}
	\caption{Low-level and high-level features learned by PaRI-Conv. In each image pair, the left shows low-level features learned at layer 2 and the right shows high-level features learned at layer 4.}\label{fig:vis_act}
\end{figure}
\subsection{Feature Activation}
We visualize the features learned by our PaRI-Conv on ModelNet40. As shown in Figure~\ref{fig:vis_act}, we colourize the features according to the level of activation at layer 2 (left) and layer 4 (right). We observe that the shallower layer tends to capture low-level features such as planes (a, c, f), conical surface (b, h), circles (e, h) and corners (d, i). In deeper layer, PaRI-Conv learns high-level structures such as stairs (a), cones (b), lamp holders (c), heads (d), guitar necks and head-stocks (e), \etc.

\subsection{Dense Rotation Invariant Features}
\noindent
\textbf{Implementation Detail for Figure 6 in the Main Paper.}
We use T-SNE to map the learned features into 3D representations, which are then normalized to $[0,1]$ to represent the RGB values.

Moreover, we observe that our PaRI-Conv is also invariant to reflections, \ie, PaRI-Conv can map identical structures in the same shape (\eg, left and right wings and engines) into similar representations. However, the representations learned by DGCNN tend to be bound with the absolute location information and fail to stay invariant. We assume that the reason for this invariance of PaRI-Conv is that the reflection (similar to rotation) preserves the relative poses between neighbouring structures. Thus, the $\mathcal{W}(\mathcal{P}_r^j)$ in Equation~\ref{equ:pari-conv} remains the same, leading to invariance to reflections. We leave further exploration on this property for future work.

\noindent
\textbf{Feature Similarity.}
We further investigate the feature similarity between different points on the same shape. As shown in Figure~\ref{fig:rot_feat_vis_supp}, we visualize the similarity between the circled point and other points. Different form rotation-sensitive methods, such as DGCNN~\cite{dgcnn}, our PaRI-Conv can also map identical structures in different poses (\eg, the left, right wings and edge of the bathtub) into the same features. This indicates even on aligned data that are free of rotation perturbation, PaRI-Conv still has the potential in developing more compact networks via kernel weight sharing.

\begin{figure}[htbp]
	\centering
	\includegraphics[width=0.48\textwidth]{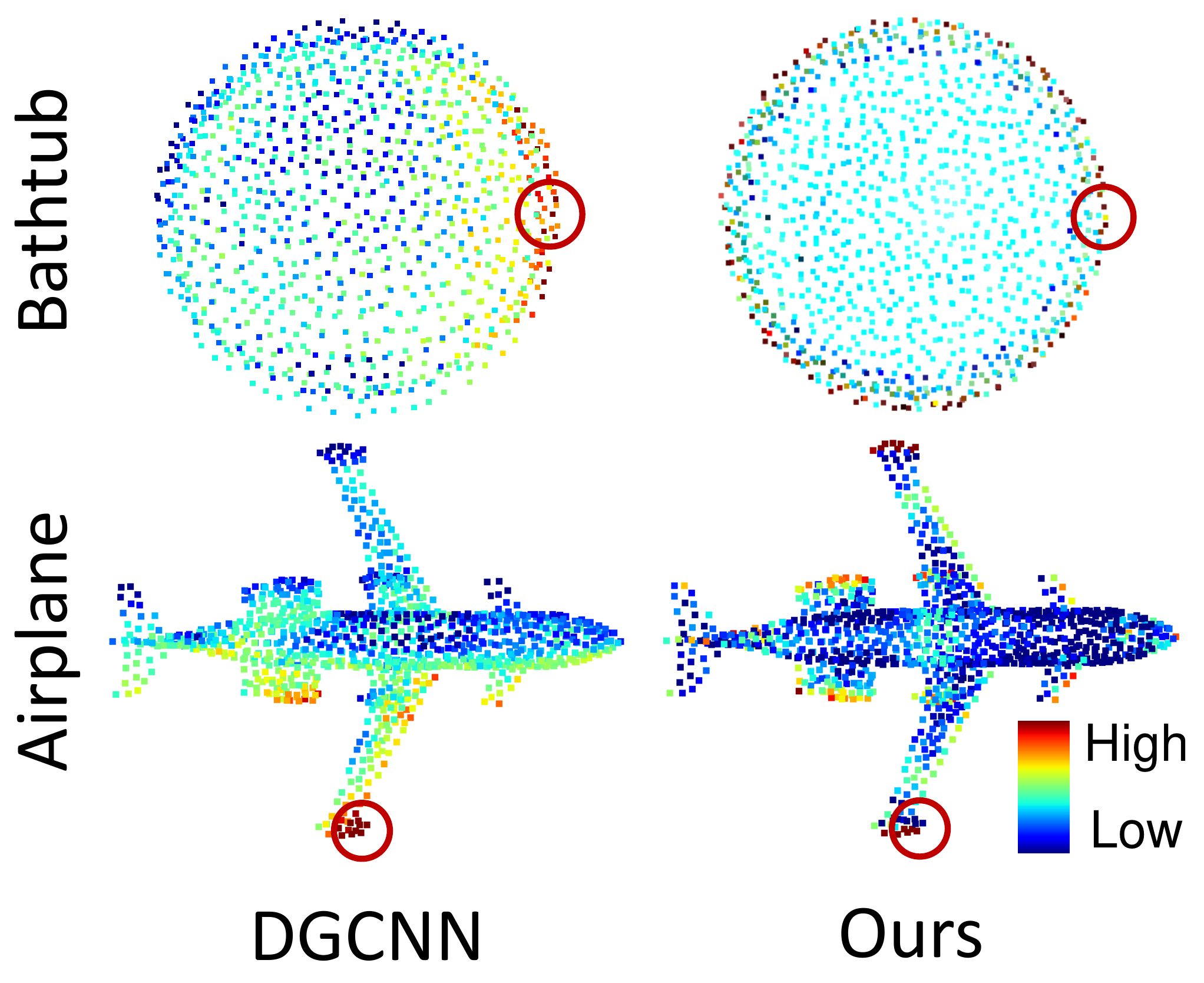}
	\caption{Feature similarity between the circled point and other points learned by DGCNN and our PaRI-Conv on ModelNet40 dataset.}\label{fig:rot_feat_vis_supp}
\end{figure}

\section{Robustness analysis}\label{sec:robust}

\subsection{Robustness to Sampling Density}

Figure~\ref{fig:noise} a) shows the results under varying density. During training, the input point clouds have 1024 points and are augmented with random point dropout. Our method can achieve reasonable performance with even half of the points and consistently outperforms Li \etal~\cite{li2021rotation}. Moreover, it also shows that applying normal as the primal axis can significantly improve the resistance to lower density.

\subsection{Robustness to Noises}
Figure~\ref{fig:noise} b) shows the results under Gaussian noises of different standard variation (Std). Here, the normals are directly extracted from noisy point clouds via PCA. Our method has certain robustness to noises and can outperform RI-GCN~\cite{NIPS2020KimLocal} when $\sigma<$ 0.04. The inferior performance under extreme noise could be attributed to the instability of current LRF, which can be addressed by applying more stable LRFs.
\begin{figure}[htbp]
	
	\centering
	\includegraphics[width=0.48\textwidth]{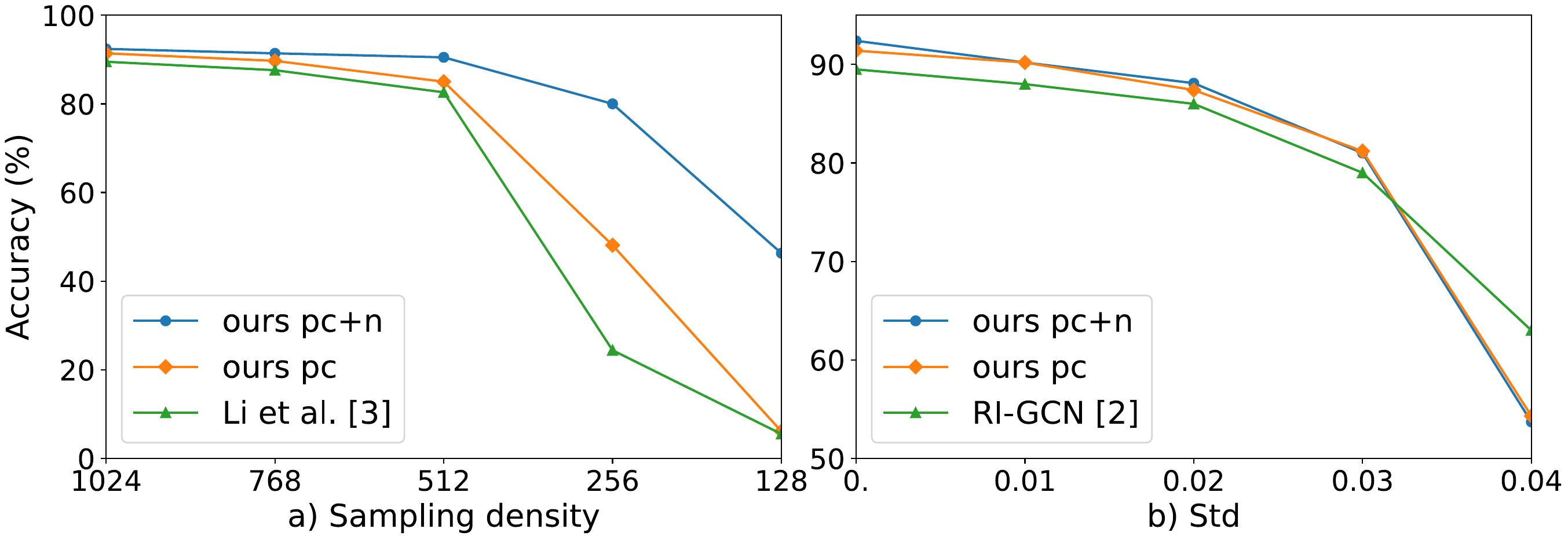}
	
	\caption{Comparison results with \textbf{a)} varying sampling density and \textbf{b)} different scales of noises on ModelNet40 dataset under $z/{\rm SO}(3)$ setting. }\label{fig:noise}
\end{figure}

\subsection{Performance Under Different Neighbour Sizes}

As shown in Table.~\ref{table:knn}, PaRI-Conv is not quite sensitive to the neighbour size $k$, and achieves the best result with $k=20$.
\begin{table}[htbp]
	\centering
	\setlength{\tabcolsep}{4mm}
	\scalebox{0.8}{
		\begin{tabular}{l|ccc}
			\hline
			& $k=10$ & $k=20$ & $k=40$\\
			\hline
			PaRI-Conv (pc) & 90.5  & \textbf{91.4} & 90.4 \\
			PaRI-Conv (pc+normal) & 91.5 & \textbf{92.4} & 91.7 \\
			\hline
	\end{tabular}}
	\caption{Performance under different neighbour size $k$ on ModelNet40 dataset. Results are evaluated under $z/SO(3)$ setting.}
	\label{table:knn}
\end{table}

\section{Detailed Results on the Discussion}\label{sec:discuss}

Here, we provide detailed comparison results with more recent state-of-the-art methods in Table~\ref{table:aligned}. Similar to other methods, we also directly take 3D coordinates as input. When normal is available, different from PointASNL~\cite{yan2020pointasnl} that directly uses normal as additional input attributes, we only utilize them to construct a more stable LRF for the extraction of APPF. Surprisingly, with normal as a stable axis, our PaRI-Conv achieves $\textbf{93.8\%}$ overall accuracy, which outperforms recent point convolution method PAConv~\cite{xu2021paconv} and powerful transformer based method~\cite{zhao2021point}. We believe that above results reveal the significance of pose information, which has been greatly overlooked by current point cloud analysis techniques. Moreover, this also demonstrates that the proposed PaRI-Conv is an effective operator in capturing pose-variant geometric structures.

\begin{table}[htbp]
	\centering
	\begin{tabular}{l|ccc}
		\hline
		Methods & input & mAcc & OA\\
		\hline
		DGCNN~\cite{dgcnn} & pc & 90.2 & 92.9 \\
		PointASNL~\cite{yan2020pointasnl} & pc & - & 92.9 \\
		PointASNL~\cite{yan2020pointasnl} & pc+normal & - & 93.2 \\
		PAConv~\cite{xu2021paconv} & pc & -  & 93.6  \\
		AdaptConv~\cite{zhou2021adaptive}  & pc & 90.7  & 93.4 \\
		PointTransformer~\cite{zhao2021point} & pc & 90.6 & 93.7 \\
		\hline
		Ours & pc & 90.4 & 93.2 \\
		Ours (normal in LRF) & pc & \textbf{91.3} & \textbf{93.8} \\
		\hline
	\end{tabular}
	\vspace{-5pt}
	\caption{Performance comparison with state-of-the-art rotation-sensitive methods on ModelNet40 dataset. Rotation perturbation is not applied. Our proposed PaRI-Conv directly takes 3D coordinates as input.}
	\label{table:aligned}
	\vspace{-10pt}
\end{table}